%% file: main.tex
\documentclass[runningheads]{llncs}

\usepackage[misc]{ifsym}

\usepackage[ruled, vlined, nofillcomment, linesnumbered]{algorithm2e}
\usepackage{cite}
\usepackage{url}
\usepackage{amsmath}
\usepackage{amssymb}
\usepackage{graphicx}
\usepackage{subcaption}
\usepackage[english]{babel}
\usepackage{booktabs}
\usepackage{verbatim}
\usepackage{float}
\usepackage{xspace}
\usepackage{tikz}
\usepackage{pgfplots}
\usepackage{pgfplotstable}
\usepackage[misc]{ifsym}

\usepackage[square, numbers,compress]{natbib}

\renewcommand{\refname}{References}
\makeatletter
\renewcommand\bibsection{%
  \section*{{\refname}\@mkboth{\refname}{\refname}}%
}%
\makeatother

\input{defines}

\begin{document}
\title{Approximation algorithms for confidence bands for time series}
\author{Nikolaj Tatti \orcidID{0000-0002-2087-5360}}
\institute{University of Helsinki, Finland, \email{nikolaj.tatti@helsinki.fi}}

\tocauthor{Nikolaj~Tatti}
\toctitle{Approximation algorithms for confidence bands for time series}

\maketitle              
\begin{abstract}

Confidence intervals are a standard technique for analyzing data.
When applied to time series,
confidence intervals are computed for each time point separately.
Alternatively, we can compute confidence bands, where we are
required to find the smallest area enveloping $k$ time series, where $k$ is a user
parameter. Confidence bands can be then used to detect abnormal time series, not
just individual observations within the time series.
We will show that despite being an \np-hard problem it is possible to find
optimal confidence band for some $k$. We do this by considering a different
problem: discovering regularized bands, where we minimize the envelope area minus the number of included time
series weighted by a parameter $\alpha$. Unlike normal confidence bands we can
solve the problem exactly by using a minimum cut. By varying $\alpha$ we can obtain solutions
for various $k$. If we have a constraint $k$ for which we cannot find appropriate $\alpha$, we
demonstrate a simple algorithm that yields $\bigO{\sqrt{n}}$ approximation guarantee by connecting
the problem to a minimum $k$-union problem. This connection also implies that we cannot
approximate the problem better than $\bigO{n^{1/4}}$ under some (mild) assumptions. Finally, we consider a variant
where instead of minimizing the area we minimize the maximum width.
Here, we demonstrate a simple 2-approximation algorithm and show that we cannot achieve
better approximation guarantee.

\end{abstract}

\input{introduction}
\input{prel}
\input{reg}
\input{sum}
\input{max}
\input{related}
\input{exp}
\input{conclusions}

\bibliographystyle{splncsnat}
\bibliography{bibliography}

\end{document}

%% file: defines.tex
\SetKw{Output}{output}

\newcommand{\set}[1]{\left\{#1\right\}}
\newcommand{\pr}[1]{\left(#1\right)}
\newcommand{\fpr}[1]{\mathopen{}\left(#1\right)}

\newcommand{\abs}[1]{{\left|#1\right|}}
\newcommand{\norm}[1]{\left\|#1\right\|}
\newcommand{\floor}[1]{\left\lfloor#1\right\rfloor}

\newcommand{\enpr}[2]{\pr{#1 ,\ldots , #2}}

\newcommand{\real}{\mathbb{R}}
\newcommand{\np}{\textbf{NP}}
\newcommand{\poly}{\textbf{P}}

\newcommand{\funcdef}[3]{{#1}:{#2} \to {#3}}
\newcommand{\define}{\leftarrow}

\newcommand{\fm}[1]{{\mathcal{#1}}}

\DeclareMathOperator*{\argmin}{arg\,min}
\DeclareMathOperator*{\argmax}{arg\,max}

\DeclareRobustCommand{\dispfunc}[2]{%
	\ensuremath{%
		\ifthenelse{\equal{#2}{}}%
			{\mathit{#1}}%
			{\mathit{#1}\fpr{#2}}}}

\newcommand{\lbound}[1]{\dispfunc{\ell b}{#1}}
\newcommand{\ubound}[1]{\dispfunc{ub}{#1}}
\newcommand{\onescore}[1]{\dispfunc{s_1}{#1}}
\newcommand{\infscore}[1]{\dispfunc{s_\infty}{#1}}
\newcommand{\regscore}[1]{\dispfunc{s_\mathit{reg}}{#1}}

\newcommand{\bigO}[1]{\dispfunc{\mathcal{O}}{#1}}

\newcommand{\dtname}[1]{\textsl{#1}}

\newcommand{\alginf}{\textsc{FindInf}\xspace}
\newcommand{\algreg}{\textsc{EnumReg}\xspace}
\newcommand{\algsum}{\textsc{FindSum}\xspace}
\newcommand{\algpeel}{\textsc{Peel}\xspace}
\newcommand{\alginfs}{\textsc{Inf}\xspace}
\newcommand{\algsums}{\textsc{Sum}\xspace}

\newcommand{\prbinf}{\textsc{InfBand}\xspace}
\newcommand{\prbreg}{\textsc{RegBand}\xspace}
\newcommand{\prbsum}{\textsc{SumBand}\xspace}
\newcommand{\prbclique}{\textsc{Clique}\xspace}
\newcommand{\prbminunion}{\textsc{MinUnion}\xspace}

\SetKw{And}{and}
\SetKw{Or}{or}
\SetKw{Not}{not}

\pgfdeclarelayer{background}
\pgfdeclarelayer{foreground}
\pgfsetlayers{background,main,foreground}

\definecolor{yafaxiscolor}{rgb}{0.3, 0.3, 0.3}

\definecolor{yafcolor1}{rgb}{0.4, 0.165, 0.553}
\definecolor{yafcolor2}{rgb}{0.949, 0.482, 0.216}
\definecolor{yafcolor3}{rgb}{0.47, 0.549, 0.306}
\definecolor{yafcolor4}{rgb}{0.925, 0.165, 0.224}
\definecolor{yafcolor5}{rgb}{0.141, 0.345, 0.643}
\definecolor{yafcolor6}{rgb}{0.965, 0.933, 0.267}
\definecolor{yafcolor7}{rgb}{0.627, 0.118, 0.165}
\definecolor{yafcolor8}{rgb}{0.878, 0.475, 0.686}

\newlength{\yafaxispad}
\newlength{\yaftlpad}
\newlength{\yaflabelpad}
\newlength{\yafaxiswidth}
\newlength{\yafticklen}
\makeatletter
\def\pgfplots@drawtickgridlines@INSTALLCLIP@onorientedsurf#1{}
\makeatother

\newcommand{\yafdrawaxis}[4]{
	\pgfplotstransformcoordinatex{#1}\let\xmincoord=\pgfmathresult 
	\pgfplotstransformcoordinatex{#2}\let\xmaxcoord=\pgfmathresult 
	\pgfplotstransformcoordinatey{#3}\let\ymincoord=\pgfmathresult 
	\pgfplotstransformcoordinatey{#4}\let\ymaxcoord=\pgfmathresult 
	\pgfsetlinewidth{\yafaxiswidth} 
	\pgfsetcolor{yafaxiscolor}
	\pgfpathmoveto{\pgfpointadd{\pgfpointadd{\pgfplotspointrelaxisxy{0}{0}}{\pgfqpointxy{\xmincoord}{0}}}{\pgfqpoint{-0.5\yafaxiswidth}{\yafaxispad}}}
	\pgfpathlineto{\pgfpointadd{\pgfpointadd{\pgfplotspointrelaxisxy{0}{0}}{\pgfqpointxy{\xmaxcoord}{0}}}{\pgfqpoint{0.5\yafaxiswidth}{\yafaxispad}}}
	\pgfpathmoveto{\pgfpointadd{\pgfpointadd{\pgfplotspointrelaxisxy{0}{0}}{\pgfqpointxy{0}{\ymincoord}}}{\pgfqpoint{\yafaxispad}{-0.5\yafaxiswidth}}}
	\pgfpathlineto{\pgfpointadd{\pgfpointadd{\pgfplotspointrelaxisxy{0}{0}}{\pgfqpointxy{0}{\ymaxcoord}}}{\pgfqpoint{\yafaxispad}{0.5\yafaxiswidth}}}
	\pgfusepath{stroke}
}

\pgfplotscreateplotcyclelist{yaf}{%
{yafcolor5,mark options={scale=0.75},mark=o}, 
{yafcolor2,mark options={scale=0.75},mark=square},
{yafcolor3,mark options={scale=0.75},mark=triangle},
{yafcolor4,mark options={scale=0.75},mark=o},
{yafcolor1,mark options={scale=0.75},mark=o},
{yafcolor6,mark options={scale=0.75},mark=o},
{yafcolor7,mark options={scale=0.75},mark=o},
{yafcolor8,mark options={scale=0.75},mark=o}} 

\pgfkeys{/pgf/number format/.cd,1000 sep={\,}}

\pgfplotsset{axis y line=left, axis x line=bottom,
	tick align=outside,
	tickwidth=\yafticklen,
	clip = false,
    x axis line style= {-, line width = 0pt, color=black!0},
    y axis line style= {-, line width = 0pt, color=black!0},
    x tick style= {line width = \yafaxiswidth, color=yafaxiscolor, yshift = \yafaxispad},
    y tick style= {line width = \yafaxiswidth, color=yafaxiscolor, xshift = \yafaxispad},
    x tick label style = {font=\scriptsize, yshift = \yaftlpad, inner xsep = 0pt},
    y tick label style = {font=\scriptsize, xshift = \yaftlpad},
    every axis y label/.style = {at = {(ticklabel cs:0.5)}, rotate=90, anchor=center, font=\scriptsize, yshift = -\yaflabelpad, inner sep = 0pt},
    every axis x label/.style = {at = {(ticklabel cs:0.5)}, anchor=center, font=\scriptsize, yshift = \yaflabelpad},
    x tick label style = {font=\scriptsize, yshift = 1pt},
    grid = major,
    major grid style  = {dash pattern = on 1pt off 3 pt},
	every axis plot post/.append style= {line width=\yafaxiswidth} ,
	legend cell align = left,
	legend style = {inner sep = 1pt, cells = {font=\scriptsize}},
	legend image code/.code={%
		\draw[mark repeat=2,mark phase=2,#1] 
		plot coordinates { (0cm,0cm) (0.15cm,0cm) (0.3cm,0cm) };%
	} 
}

%% file: introduction.tex
\section{Introduction}

Confidence intervals are a common tool to summarize the underlying
distribution, and to indicate outlier behaviour. In this paper we will
study the problem of computing confidence intervals for time series.

\citet{korpela2014confidence} proposed a notion for
computing confidence intervals: instead of computing point-wise 
confidence intervals, the authors propose computing confidence bands.
More formally, given $n$ time series $T$, we are asked to find $k$ time series $U \subseteq T$
that minimize the envelope area, that is, the sum $\sum_i \pr{\max_{t \in U} t(i)} - \pr{\min_{t \in U} t(i)}$.
The benefit, as argued by \citet{korpela2014confidence}, of using
confidence bands instead of point-wise confidence intervals
is better family-wise error control: if we were to use point-wise intervals
we can only say that a time series \emph{at some fixed point} is an outlier
and require a correction for multiple testing (such as Bonferroni correction)
if we want to state with a certain probability that the \emph{whole} time series is normal.

In this paper we investigate the approximation algorithms for finding
confidence bands. While \citet{korpela2014confidence} proved that
finding the optimal confidence band is an \np-hard problem, they did not
provide any approximation algorithms nor any inapproximability results.

We will first show that despite being an \np-hard problem, we can
solve the problem for \emph{some} $k$. We do this by considering a different
problem, where instead of having a hard constraint we have an objective
function that prefers selecting time series as long as they do not increase
the envelope area too much. The objective depends on the parameter $\alpha$,
larger values of $\alpha$ allow more increase in the envelope area.
We will show that this problem can be solved exactly in polynomial time and that 
each $\alpha$ correspond to a certain value of $k$. We will show that there are
at most $n + 1$ of such bands, and that
we can discover all of them in polynomial time by varying $\alpha$.

Next, we provide a simple algorithm for approximating confidence bands by
connecting the problem to the weighted $k$-\prbminunion problem. We will provide a variant
of an algorithm by~\citet{chlamtac2018densest} that yields $\sqrt{n} + 1$
guarantee. We also argue that---under certain conjecture---we cannot approximate
the problem better than  $\bigO{n^{1/4}}$. 

Finally, we consider a variant of the problem where instead of minimizing
the envelope area, we minimize the width of the envelope, that is, we minimize
the maximum difference between the envelope boundaries. We show that a simple
algorithm can achieve 2-approximation. This approximation provides interesting
contrast to the inapproximability results when minimizing the envelope area.
Surprisingly this guarantee is tight: we will also show
that the there is no polynomial-time algorithm with smaller guarantee unless $\poly=\np$.

The remainder of the paper is organized as follows. We define the optimization problems
formally in Section~\ref{sec:prel}. We solve the regularized band problem in Section~\ref{sec:reg},
approximate minimization of envelope area in Section~\ref{sec:sum}, and
approximate minimization of envelope width in Section~\ref{sec:max}.
Section~\ref{sec:related} is devoted to the related work. We present our experiments
in Section~\ref{sec:exp} and conclude with discussion in Section~\ref{sec:conclusions}.

%% file: prel.tex
\section{Preliminaries and problem definitions} \label{sec:prel}

Assume that we are given time series $T$ with each time series
$\funcdef{f}{D}{\real}$ mapping from domain $D$ to a real number. We will often
write $n = \abs{T}$ to be the number of given time series, and $m = \abs{D}$ to
mean the size of the domain.

Given a set of time series $T$, we define the upper and lower \emph{envelopes} as
\[
	\ubound{T, i} = \max_{t \in T} t(i) \quad\text{and}\quad
	\lbound{T, i} = \min_{t \in T} t(i) \quad .
\]

Our main goal is to find $k$ time series that minimize the envelope area.
\begin{problem}[\prbsum]
Given a set of $n$ time series $T = \enpr{t_1}{t_n}$, an integer $k \leq n$, and a time series $x \in T$
find $k$ time series $U \subseteq T$ containing $x$ minimizing
\[
	\onescore{U} = \sum_{i} \ubound{U, i} - \lbound{U, i}\quad.
\]
\end{problem}
We will refer to $U$ as \emph{confidence bands}. 

Note that the we also require that we must specify at least one sequence $x \in
T$ that must be included in the input whereas the original definition of the
problem given by~\citet{korpela2014confidence}  did not require specifying $x$.  As we will see later, this requirement
simplifies the computational problem.
On the other hand, if we do not have $x$ at hand, then we can either test every
$t \in T$ as $x$, or we can use the mean or the median of $T$. We will use the
latter option as it does not increase the computational complexity and at the
same time is a reasonable assumption. Note that in this case most likely $x \notin T$,
so we define $T' = T \cup \set{x}$, increase $k' = k + 1$, and solve \prbsum for $T'$ and $k'$ instead.

We can easily show that the area function $\onescore{\cdot}$ is a submodular
function for all non-empty subsets, that is,
\[
	\onescore{U \cup \set{t}} - \onescore{U} \leq \onescore{W \cup \set{t}} - \onescore{W},
\]
where $U \supseteq W \neq \emptyset$. In other words, adding $t$ to a larger set $U$ increases the cost
less than adding $t$ to $W$.

We also consider a variant of \prbsum where instead of minimizing the area of the envelope, we will minimize
the maximum width.
\begin{problem}[\prbinf]
Given a set of $n$ time series $T = \enpr{t_1}{t_n}$, an integer $k \leq n$, and a time series $x \in T$,
find $k$ time series $U \subseteq T$ containing $x$ minimizing
\[
	\infscore{U} = \max_{i} \ubound{U, i} - \lbound{U, i}\quad.
\]
\end{problem}

We will show that we can 2-approximate \prbinf and that the ratio is tight. 

Finally, we consider a regularized version of \prbsum, where instead of requiring that the set has a minimum size $k$,
we add a term $- \alpha \abs{U}$ into the objective function. In other words, we will favor larger sets as long as
the area $\onescore{U}$ does not increase too much.

\begin{problem}[\prbreg]
Given a set of $n$ time series $T = \enpr{t_1}{t_n}$, a number $\alpha > 0$, and a time series $x \in T$,
find a subset $U \subseteq T$ containing $x$ minimizing
\[
	\regscore{U; \alpha} = \onescore{U} - \alpha \abs{U} \quad.
\]
In case of ties, use $\abs{U}$ as a tie-breaker, preferring larger values.
\end{problem}

We refer to the solutions of \prbreg as regularized bands.
It turns out that \prbreg can be solved in polynomial time.
Moreover, the solutions we obtain from \prbreg will be useful for approximating
\prbsum.

%% file: reg.tex
\section{Regularized bands}\label{sec:reg}

In this section we will list useful properties of
of \prbreg, show how can we solve \prbreg in polynomial time for a single $\alpha$,
and finally demonstrate how we can discover \emph{all} regularized bands by varying $\alpha$.

\subsection{Properties of regularized bands}

Our first observation is that the output of \prbreg also solves \prbsum for
certain size constraints.

\begin{proposition}
\label{prop:optimal}
Assume time series $T$ and $\alpha > 0$. Let $U$ be a solution to $\prbreg(\alpha)$.
Then $U$ is also a solution for $\prbsum$ with $k = \abs{U}$.
\end{proposition}

The proof of this proposition is trivial and is omitted.

Our next observation is that the solutions to \prbreg form a chain.

\begin{proposition}
\label{prop:chain}
Assume time series $T$ and $0 < \alpha < \beta$.
Let $V$ be a solution to $\prbreg(T, \alpha)$ and let $U$ be a solution to $\prbreg(T, \beta)$.
Then $V \subseteq U$. 
\end{proposition}

\begin{proof}
Assume otherwise.
Let $W = V \setminus U$. Due to the optimality of $V$,
\[
	0 \geq  \regscore{V; \alpha} -  \regscore{V \cap U; \alpha}
	= \onescore{V} - \onescore{V \cap U} - \alpha \abs{W}\quad.
\]
Since $\onescore{}$ is a submodular function, we have
\[
	\onescore{V} - \onescore{V \cap U} =
	\onescore{W \cup (V \cap U)} - \onescore{V \cap U}
	\geq \onescore{W \cup U} - \onescore{U}\quad.
\]
Combining these inequalities leads to
\[
\begin{split}
	0 & \geq \onescore{W \cup U} - \onescore{U} - \alpha \abs{W} \\
	& \geq \onescore{W \cup U} - \onescore{U} - \beta \abs{W} \\
	& = \regscore{W \cup U; \beta} - \regscore{U; \beta},
\end{split}
\]
which contradicts the optimality of $U$.\qed
\end{proof}

This property is particularly useful as it allows clean visualization: the
envelopes resulting from  different values of $\alpha$ will not intersect.
Moreover, it allows us to stored all regularized bands by simply storing, per each time series,
the index of the largest confidence band containing the time series.

Interestingly, this result does not hold for $\prbsum$.

\begin{example}
Consider 4 constant time series $t_1 = 0$, $t_2 = -1$ and $t_3 = t_4 = 2$.
Set the seed time series $x = t_1$.
Then the solution for \prbsum with $k = 2$ is $\set{t_1, t_2}$ and the solution
\prbsum with $k = 3$ is $\set{t_1, t_3, t_4}$.
\end{example}

\subsection{Computing regularized band for a single $\alpha$}

Our next step is to solve $\prbreg$ in polynomial time.
Note that since $\onescore{\cdot}$ is submodular, then so is $\regscore{\cdot}$.
Minimizing submodular function is solvable in polynomial-time~\citep{schrijver2000combinatorial}.
Solving \prbreg using a generic solver for minimizing submodular functions is slow, so instead
we will solve the problem by reducing it to a minimum cut problem. In such a problem, we are given
a weighted directed graph $G = (V, E, W)$, two nodes, say $\theta, \eta \in V$, and ask to partition $V$ into $X \cup Y$
such that $\theta \in X$ and $\eta \in Y$ minimizing
the total weight of edges from $X$ to $Y$.

In order to define $G$ we need several definitions.
Assume we are given $n$ time series $T$, a real number $\alpha$
and a seed time series $x \in T$. Let $m$ be the size of the domain.
For $i \in [m]$, we define $p_i = \set{t_j(i) \mid j \in [n]}$ to be the set (with no duplicates) sorted, smallest values first.
In other words, $p_{ij}$ is the $j$th smallest distinct observed value in $T$ at $i$. Let $P$ be the collection of all $p_i$.

We also define $c_{ij}$ to be the number of time series at $i$ smaller than or equal to $p_{ij}$, that is,
	$c_{ij} = \abs{\set{\ell \in [n] \mid t_\ell(i) \leq p_{ij}}}$. 
We also write $c_{i0} = 0$.

We are now ready to define our graph.
We define a weighted directed graph $G = (V, E, W)$ as follows. The nodes $V$ have
three sets $A$, $B$, and $C$. The set $A$ has $\abs{P}$ nodes, a node $a_{ij} \in A$
corresponding to each entry $p_{ij} \in P$.
The set $B = \set{b_j}$ has $n$ nodes, and the set $C$ has two nodes, $\theta$ and $\eta$.
Here, $\theta$ acts as a source node and $\eta$ acts as a terminal node.

The edges and the weights are as follows:
For each $a_{ij} \in A$ such that $p_{ij} > x(i)$, we add an edge $(a_{i(j - 1)}, a_{ij})$ with the weight
\[
	w(a_{i(j - 1)}, a_{ij}) = n - c_{i(j - 1)} + \frac{m}{\alpha}(p_{i(j - 1)} - x(i))\quad.
\]
For each $a_{ij} \in A$ such that $p_{ij} < x(i)$, we add an edge $(a_{i(j + 1)}, a_{ij})$ with the weight
\[
	w(a_{i(j + 1)}, a_{ij}) = c_{ij} + \frac{m}{\alpha}(x(i) - p_{i(j + 1)})\quad.
\]
For each $a_{ij} \in A$ such that $p_{ij} = x(i)$, we add an edge $(\theta, a_{ij})$ with the weight $\infty$.
For each $i \in [m]$ and $\ell = \abs{p_i}$, we add two edges $(a_{i\ell}, \eta)$ and $(a_{i1}, \eta)$ with the weights
\[
	w(a_{i\ell}, \eta) = \frac{m}{\alpha}(p_{i\ell} - x(i)) \quad\text{and}\quad
	w(a_{i1}, \eta) = \frac{m}{\alpha}(x(i) - p_{i1})\quad.
\]

In addition, for each $i \in [m]$, $\ell \in [n]$, let $j$ be such that $p_{ij} = t_\ell(i)$ and 
define two edges $(a_{ij}, b_\ell)$ and $(b_\ell, a_{ij})$ with the weights,
\begin{align*}
	w(a_{ij}, b_\ell) & = 1 &
	w(b_\ell, a_{ij}) & = \infty\quad.
\end{align*}

Our next proposition states the minimum cut of $G$ also minimizes \prbreg.

\begin{proposition}
\label{prop:cut}
Let $X, Y$ be a $(\theta, \eta)$-cut of $G$ with the optimal cost.
Define $f(i) = \min_j \set{p_{ij} \mid a_{ij} \in X}$
and $g(i) = \max_j \set{p_{ij} \mid a_{ij} \in X}$.

Then the cost of the cut is equal to
\[
	nm - m\abs{\set{j \mid b_j \in X}} + \frac{m}{\alpha}\sum_{i} g(i) - f(i)\quad.
\]
Moreover, if $b_\ell \in X$, then $g(i) \leq b_\ell(i) \leq f(i)$, for all $i$.
\end{proposition}

\begin{proof}
The last claim follows immediately as otherwise there is a cross-edge with infinite cost
making the cut suboptimal.

Define
$u(i) = \argmin_j \set{p_{ij} \mid a_{ij} \in X}$
and
$v(i) = \argmax_j \set{p_{ij} \mid a_{ij} \in X}$
to be the indices yielding $f$ and $g$. 
Define also
\[
	d_i = \abs{\set{j \mid u(i) \leq t_j(i) \leq v(i)}} = c_{iv(i)} - c_{i(u(i) - 1)}
\]
to be the number of time series between $u(i)$ and $v(i)$ at $i$.

Note that $a_{ij} \in X$ whenever $u(i) \leq j \leq v(j)$ as otherwise we can move $a_{ij}$
to $X$ and decrease the cost.

The cut consists of the cross-edges originating from $a_{iv(i)}$ and $a_{iu(i)}$,
and cross-edges between $A$ and $B$. 
The cost of the former is equal to
\[
\begin{split}	
	& \sum_i n - c_{iv(i)} + m\frac{p_{iv(i)} - x(i)}{\alpha} + c_{i(u(i) - 1)} + m\frac{x(i) - p_{iu(i)}}{\alpha} \\
	& \qquad = nm + \sum \frac{m}{\alpha} (g(i) - f(i)) - \sum_i d_i
\end{split}	
\]
while the cost of the latter is
\[
\begin{split}	
	 \sum_i \abs{\set{j \mid a_{ij} \in X, b_j \notin X}}
	 & = \sum_i \abs{\set{j \mid u(i) \leq t_j(i) \leq v(i) \in X, b_j \notin X}} \\
	 & = \sum_i d_i - \abs{\set{j \mid u(i) \leq t_j(i) \leq v(i) \in X, b_j \in X}} \\
	 & = \sum_i d_i - m\abs{\set{j \mid b_j \in X}} \quad.\\
\end{split}	
\]
Combining the two equations proves the claim.\qed
\end{proof}

\begin{corollary}
Let $U'$ be the solution to $\prbreg(\alpha)$.
Let $(X, Y)$ be a minimum $(\theta, \eta)$-cut of $G$. Set $U = \set{t_\ell \mid b_\ell \in X}$.
Then $\regscore{U; \alpha} = \regscore{U'; \alpha}$.
\end{corollary}

\begin{proof}
Proposition~\ref{prop:cut} states that the cost of the minimum cut
is $nm + \frac{m}{\alpha}\regscore{U; \alpha}$.

Construct a cut $(X', Y')$ from $U'$ by setting $X'$ to be the nodes from $A$ and $B$ 
that correspond to the time series $U'$. The proof of Proposition~\ref{prop:cut} 
now states that the cut is equal to $nm + \frac{m}{\alpha}\regscore{U'; \alpha}$.

The optimality of $(X, Y)$ proves the claim.\qed
\end{proof}

We may encounter a pathological case, where we have multiple cuts with the same optimal cost.
\prbreg requires that in such case we use largest solution. This can be enforced
by modifying the weights: first scale the weights so that they are all multiples of $nm + 1$, then
add 1 to the weight of each $(\theta, \alpha_{ij})$. The cut with the modified graph yields
the largest band with the optimal cost.

The constructed graph $G$ has $\bigO{nm}$ nodes and $\bigO{nm}$ edges.
Consequently, we can compute the minimum cut in $\bigO{(nm)^2}$ time~\citep{orlin2013max}.
In practice, solving minimum cut is much faster.

\subsection{Computing all regularized bands}

Now that we have a method for solving $\prbreg(\alpha)$ for a fixed $\alpha$, we would like to
find solutions for all $\alpha$.
Note that Proposition~\ref{prop:chain} states that we can have at most $n + 1$ different bands.

We can enumerate the bands with the divide-and-conquer approach given in Algorithm~\ref{alg:reg}.
Here, we are given two, already discovered, regularized bands $U \subsetneq V$, and we try to find a middle band $W$ with $U \subsetneq W \subsetneq V$.
If $W$ exists, we recurse on both sides. To enumerate all bands, we start with $\algreg(\set{x}, V)$.

\begin{algorithm}
\caption{$\algreg(U, V)$ finds all regularized bands between $U$ and $V$}
\label{alg:reg}
$\gamma \define \frac{\onescore{V} - \onescore{U}}{\abs{V} - \abs{U}} - \frac{\Delta}{n^2}$\;
$W \define $ solution to $\prbreg(\gamma)$\;
\If {$U \neq W$} {
	report $W$\;
	$\algreg(U, W)$;
	$\algreg(W, V)$\;
}
\end{algorithm}

The following proposition proves the correctness of the algorithm:
during each split we will always find a new band if such exist.

\begin{proposition}
\label{prop:split}
Assume time series $T$ with $n$ time series.
Let $\set{U_i}$ be all the possible regularized confidence bands ordered using inclusion.
Define 
\[
	\Delta = \min \set{\abs{t(i) - u(i)} \mid t, u \in T, i, t(i) \neq u(i)} \quad.
\]
Let $i < j$ be two integers and define
\[
	\gamma = \frac{\onescore{U_j} - \onescore{U_i}}{\abs{U_j} - \abs{U_i}} - \frac{\Delta}{n^2}\quad.
\]
Let $U_\ell$ be the solution for $\prbreg(\gamma)$.
Then $i \leq \ell < j$.
If $j > i + 1$, then $i < \ell$, otherwise $\ell = i$.
\end{proposition}

For simplicity, let us define $f(x, y) = \frac{\onescore{U_y} - \onescore{U_x}}{\abs{U_y} - \abs{U_x}}$.

In order to prove the result we need the following technical lemma.

\begin{lemma}
\label{lem:monotone}
Assume time series $T$ with $n$ time series.
Let $\set{U_i}$ be all the possible regularized confidence bands ordered using inclusion.
Let $\alpha > 0$. Let $U_i$ be the solution for $\prbreg(\alpha)$.
Then
	$f(i - 1, i) \leq \alpha < f(i, i + 1)$.
\end{lemma}

\begin{proof}
Due to the optimality of $U_i$,
\[
	\onescore{U_i}  - \alpha \abs{U_i} = \regscore{U_i; \alpha} < \onescore{U_{i + 1}} - \alpha \abs{U_{i + 1}}\quad.
\]
Solving for $\alpha$ gives us the right-hand side of the claim.
Similarly, 
\[
	\onescore{U_i}  - \alpha \abs{U_i} = \regscore{U_i; \alpha} \leq \onescore{U_{i - 1}} - \alpha \abs{U_{i - 1}}\quad.
\]
Solving for $\alpha$ gives us the left-hand side of the claim.\qed
\end{proof}

\begin{proof}[of Proposition~\ref{prop:split}]
It is straightforward to see that
Lemma~\ref{lem:monotone} implies that $f(a, b) \leq f(x, y)$
for $a \leq x$ and $b \leq y$. Moreover, the equality holds only if $x = a$ and $y = b$,
in other cases $f(a, b) +  \frac{\Delta}{n^2} \leq f(x, y)$.

If $\ell \geq j$, then Lemma~\ref{lem:monotone} states that $f(i, j) \leq f(\ell - 1, \ell) \leq \gamma$,
which contradicts the definition of $\gamma$. Thus $\ell < j$.

Since $f(i, j) - f(i - 1, i) \geq  \frac{\Delta}{n^2}$,
we have $f(i - 1, i) \leq \gamma$.
If $\ell < i$, then Lemma~\ref{lem:monotone} states that $\gamma < f(i - 1, i)$,
which is a contradiction.
Thus, $\ell \geq i$.

If $j = i + 1$, then immediately $\ell = i$.

Assume that $j > i + 1$.
Since $f(i, j) - f(i, i + 1) \geq  \frac{\Delta}{n^2}$,
we have $f(i, i + 1) \leq \gamma$.
If $\ell = i$, then according to Lemma~\ref{lem:monotone} $\gamma < f(i, i + 1)$, which is a contradiction.
Thus, $\ell > i$.
\qed
\end{proof}

Lemma~\ref{lem:monotone} reveals an illuminating property of regularized bands,
namely each band minimizes the ratio of additional envelope area and the number
of new time series. 

\begin{proposition}
Let $U$ be a regularized band.
Define $g(X) = \frac{\onescore{X} - \onescore{U}}{\abs{X} - \abs{U}}$.
Let $V \supsetneq U$ be the adjacent regularized band.
Then $g(V) = \min_{X \supsetneq U} g(X)$.
\label{prop:sparsest}
\end{proposition}

\begin{proof}
Let $O = \argmin_{X \supsetneq U} g(X)$, and set $\beta = g(O)$.
We will prove that $g(V) \leq \beta$.
Let $W = \prbreg(\beta)$.
Let $\alpha$ be the parameter for which $U = \prbreg(\alpha)$. 
Assume that $\alpha \geq \beta$. We can rewrite the equality $\beta = g(O)$ as
\[
    0 = \regscore{O; \beta} - \regscore{U; \beta}
    \geq \regscore{O; \alpha} - \regscore{U; \alpha},
\]
which violates the optimality of $U$. Thus $\alpha < \beta$.
Proposition~\ref{prop:chain} states that $U \subseteq W$. Moreover, due to submodularity,
\[
\regscore{O \cup W; \beta} -  \regscore{W; \beta} \leq 
\regscore{O \cup U; \beta} -  \regscore{U; \beta} = 0,
\]
which due to the optimality of $W$ implies that $O \subseteq W$. Thus $W \neq U$ and $V \subseteq W$.
Lemma~\ref{lem:monotone}, possibly applied multiple times, shows that $g(V) \leq g(W) \leq \beta$.  \qed
\end{proof}

Proposition~\ref{prop:chain} states that there are at most $n + 1$ bands. Queries done by \algreg yield the same band
at most twice. Thus, \algreg performs at most $\bigO{n}$ queries, yielding computational complexity of $\bigO{n^3m^2}$.
In practice, \algreg is faster: the number of bands is significantly smaller than $n$ and the minimum cut solver scales significantly better
than $\bigO{n^2m^2}$. Moreover, we can further improve the performance with the following observation:
Proposition~\ref{prop:chain} states that when processing $\algreg(U, V)$,
the bands will be between $U$ and $V$. Hence, we can ignore the time series that are outside $V$, and we can safely replace
$U$ with its envelope $\lbound{U}$ and $\ubound{U}$.\!\footnote{We need to make
sure that the envelope is always selected. This can be done by connecting
$\theta$ to the envelope with edges of infinite weight.}

%% file: sum.tex
\section{Discovering confidence bands minimizing $\onescore{}$}\label{sec:sum}

In this section, we will study \prbsum. \citet{korpela2014confidence} showed that the problem is \np-hard. 
We will argue that we can approximate the problem and establish a (likely) lower
bound for the approximation guarantee.

\begin{algorithm}[ht!]
\caption{$\algsum(T, k, x)$, approximates \prbsum}
\label{alg:sum}
$\set{B_i} \define \algreg(\set{x}, T)$\;

$j \define $ largest index for which $\abs{B_j} \leq k$\;
\textbf{if} $\abs{B_j} \leq k - \sqrt{n}$ \textbf{then} $W \define B_{j + 1} \setminus B_j$ \textbf{else} $W \define T \setminus B_j$\;

$U \define B_j$\;
greedily add $k - \abs{U}$ entries from $W$ to $U$, minimizing $\onescore{}$ at each step\;
\Return $U$\;
\end{algorithm}

As a starting point, note that \prbsum is an instance of $k$-\prbminunion, weighted minimum $k$-union
problem. In $k$-\prbminunion we are given $n$ sets over a universe with weighted points, and
ask to select $k$ sets minimizing the weighted union. In our case, the universe is the set $P$ described
in Section~\ref{sec:reg}, the weights are the distances between adjacent points, and a set consists of all the points  
between a time series and $x$.

The \emph{unweighted} $k$-\prbminunion problem has several approximation algorithms: a simple
algorithm achieving $\bigO{\sqrt{n}}$ guarantee by~\citet{chlamtac2018densest} and an algorithm achieving lower approximation
guarantee of $\bigO{n^{1/4}}$ by~\citet{chlamtavc2017minimizing}. We will use the former algorithm due to its simplicity
and the fact that it can be easily adopted to handle weights.

The pseudo-code for the algorithm is given in Algorithm~\ref{alg:sum}.
The algorithm first looks for the largest possible regularized band, say $B_j$, whose size at most $k$.
The remaining time series are then selected greedily from a set of candidates $W$.
The set $W$ depends on how many additional time series is needed: if we need
at most $\sqrt{n}$ additional time series, we set $W$ to be the remaining time series $T \setminus B_j$,
otherwise we select the time series from the next regularized band, that is, we set $W = B_{j + 1} \setminus B_j$.

\begin{proposition}
\algsum yields $\sqrt{n} + 1$ approximation guarantee.
\end{proposition}

\begin{proof}
Let $O$ be the optimal solution for $\prbsum(k)$, and let $r = \onescore{O}$.
Let $U$ be the output of \algsum. Assume that $B_j \neq O$, as otherwise we are done.
We split the proof in two cases.

First, assume that $\abs{B_{j}} \leq k - \sqrt{n}$.
Since $\onescore{}$ is submodular we have $\onescore{O \cup B_{j}} - \onescore{B_{j}} \leq \onescore{O} - \onescore{\set{x}} = r$,
leading to
\[
	\frac{\onescore{B_{j + 1}} - \onescore{B_{j}}}{\abs{B_{j + 1}} - \abs{B_{j}}} \leq
	\frac{\onescore{O \cup B_{j}} - \onescore{B_{j}}}{\abs{O \cup B_j} - \abs{B_{j}}}
	\leq \frac{r}{k - \abs{B_{j}}} \leq \frac{r}{\sqrt{n}},
\]
where the first inequality is due to Proposition~\ref{prop:sparsest}.
Rearranging the terms gives us
\[
	\onescore{B_{j + 1}} - \onescore{B_{j}} 
	\leq \frac{r(\abs{B_{j + 1}} - \abs{B_{j}})}{\sqrt{n}} \leq r \frac{n}{\sqrt{n}} = r \sqrt{n}\quad.
\]
Finally,
\[
\begin{split}
	\onescore{U} & =  \onescore{B_{j}} + (\onescore{U} -  \onescore{B_{j}}) \\
	& \leq \onescore{B_{j}} + (\onescore{B_{j + 1}} -  \onescore{B_{j}}) 
	 \leq \onescore{B_{j}} + r \sqrt{n} 
	 \leq r (1 + \sqrt{n}), \\
\end{split} 
\]
where the last inequality is implied by Proposition~\ref{prop:optimal} and the fact that $\abs{B_{j}} \leq k$.

Assume that $\abs{B_{j}} > k - \sqrt{n}$, and let $q = k - \abs{B_j}$. Note that $q < \sqrt{n}$.
Let $c_1, \ldots, c_q$ be the additional time series added to $U$. Write $U_i = B_j \cup \set{c_1, \ldots, c_i}$.

Let $c_i'$ be the closest time series to $x$ outside $U_{i-1}$.
Note that $\onescore{\set{x, c_i'}} - \onescore{\set{x}} = \onescore{\set{x, c_i'}} \leq r$ for $i = 1, \ldots, q$ as otherwise $r$ has to be larger.
In addition, Proposition~\ref{prop:optimal} and $\abs{B_{j}} \leq k$ imply that $\onescore{B_j} \leq r$.
Consequently, 
\[
\begin{split}
	\onescore{U_q} & = \onescore{B_j} + \sum_{i = 1}^q \onescore{U_i} - \onescore{U_{i - 1}} \\
	& \leq \onescore{B_j} + \sum_{i = 1}^{q} \onescore{\set{x, c_i'}} - \onescore{\set{x}} \leq (1 + \sqrt{n}) r,
\end{split}
\]
where the first inequality is due to the submodularity of $\onescore{}$.
\qed
\end{proof}

\algsum resembles greatly the algorithm given by~\citet{chlamtac2018densest}  but has few technical
differences: we select $B_j$ as our starting point whereas the algorithm by~\citet{chlamtac2018densest} 
constructs the starting set by iteratively finding and adding sets with the
smallest average area, $\onescore{X} / \abs{X}$, that is, solving the problem given in Proposition~\ref{prop:sparsest}.\!\footnote{The original
algorithm is described using set/graph terminology but we use our
terminology to describe the differences.} Such sets can be found with a
linear program.  Proposition~\ref{prop:sparsest} implies that both approaches result in the
same set $B_j$ but our approach is faster.\!\footnote{The computational
complexity of the state-of-the-art linear program solver is $\bigO{(nm)^{2.37}
\log(nm/ \delta)}$, where $\delta$ is the relative accuracy~\citep{cohen2021solving}. We may need to solve $\bigO{n}$ such problems, leading to a total time of  $\bigO{n(nm)^{2.37}
\log(nm/ \delta)}$.}
Moreover, this modification allows us to prove a tighter approximation
guarantee: the authors prove that their algorithm yields $2\sqrt{n}$ guarantee
whereas we show that we can achieve $\sqrt{n} + 1$ guarantee.  Additionally, we
select additional time series iteratively by selecting those time series that result in
the smallest increase of the current area,
whereas the original algorithm would
simply select time series that are closest to $\set{x}$.

\citet{chlamtavc2017minimizing} argued that under some mild but technical conjecture there is no polynomial-time
algorithm that can approximate $k$-\prbminunion better than $\bigO{n^{1/4}}$.
Next we will show that we can reduce $k$-\prbminunion to \prbsum while preserving approximation.  

\begin{proposition}
If there is
an $f(n)$-approximation polynomial-time algorithm for \prbsum, then
there is 
an $f(n + 1)$-approximation polynomial-time algorithm for $k$-\prbminunion.
\end{proposition}

\begin{proof}
Assume that we are given an instance of $k$-\prbminunion with $n$ sets $\fm{S} = \enpr{S_1}{S_n}$.
Let $D = \bigcup_i S_i$ be the union of all $S_i$.

Define $T$ containing $n + 1$ time series over the domain $D$. The first $n$
time series correspond to the sets $S_i$, that is, given $i \in D$,
we set $t_j(i) = 1$ if $i \in S_j$, and $0$ otherwise.
The remaining single time series, named $x$, is set to be $0$.

Assume that we have an algorithm estimating $\prbsum(T, x, k + 1)$, and let $U$ be the output 
of this algorithm. Note that since $x \in U$, we have $\lbound{U, i} = 0$.

Let $\fm{V}$ be the subset of $\fm{S}$ corresponding to the non-zero time series
in $U$. Let $C = \bigcup_{S \in \fm{V}} S$ be the union of sets in $\fm{V}$.
Since $\lbound{U, i} = 0$, and $\ubound{U, i} = 1$ if and only if $i \in C$,
we have
$\onescore{U} = \abs{C}$.
\qed
\end{proof}

The above result implies that unless the conjecture suggested by~\citet{chlamtavc2017minimizing}
is false, we cannot approximate \prbsum better than $\bigO{n^{1/4}}$.
This proposition holds even if we replace $\onescore{\cdot}$
with an $\ell_p^p$ norm, $\sum_{i} \abs{t(i) - u(i)}^p$, where $1 \leq p < \infty$,
or any norm that reduces to hamming distance if $t$ is a binary sequence and $u$ is 0. 
Interestingly, we will show in the next section that we can achieve a tighter
approximation if we use $\infscore{}$.

%% file: max.tex
\section{Discovering confidence bands minimizing $\infscore{}$}\label{sec:max}

In this section we consider the problem \prbinf. Namely, we will show
that a straightforward algorithm 2-approximates the problem, and more interestingly
we show that the guarantee is tight.

The algorithm for $\prbinf(T, x, k)$ is simple: we select $k$ time series that
are closest to $x$ according to the norm $\norm{t(i) - x(i)}_\infty = \max_i \abs{t(i) - x(i)}$.
We will refer to this algorithm as \alginf.

It turns out that this simple algorithm yields 2-approximation guarantee.

\begin{proposition}
\alginf yields 2-approximation for \prbinf.
\end{proposition}

\begin{proof}
Let $U$ be the optimal solution for \prbinf.
Let $V$ be the result produced by \alginf.
Define $c = \max_{t \in V} \norm{t - x}_\infty$.
Then
\[
	c = \max_{t \in V} \norm{t - x}_\infty 
	\leq \max_{t \in U} \norm{u - x}_\infty 
	\leq \infscore{U},
\]
where the first inequality holds since $V$ contains the closest time series
and the second inequality holds since $x \in U$.

Let $i$ be the index such that $\infscore{V} = \ubound{V, i} - \lbound{V, i}$.
Then
\[
	\infscore{V} = \ubound{V, i} - \lbound{V, i} = (\ubound{V, i} - t(i)) +  (t(i) - \lbound{V, i}) \leq 2c\quad.
\]
Thus, $\infscore{V} \leq 2c \leq 2\infscore{U}$, proving the claim.
\qed
\end{proof}

While \alginf is trivial, surprisingly it achieves the best possible approximation
guarantee for a polynomial-time algorithm.

\begin{proposition}
There is no polynomial-time algorithm for \prbinf that yields $\alpha < 2$ approximation guarantee unless $\poly = \np$.
\end{proposition}

\begin{proof}
To prove the claim we will show that we can solve $k$-\prbclique in polynomial
time if we can $\alpha$-approximate \prbinf with $\alpha < 2$.
Since $k$-\prbclique is an \np-complete problem, this is a contradiction unless $\poly = \np$.

The goal of $k$-\prbclique is given a graph $G = (V, E)$ with $n$ nodes and $m$ edges
to detect whether there is a $k$-clique, a fully connected subgraph with $k$ nodes, in $G$. We can safely assume
that $G$ has no nodes that are fully-connected.

Fix an order for nodes $V = \enpr{v_1}{v_n}$ and let $F$ be all the edges that are not
in $E$, that is,
	$F = \set{(v_x, v_y) \mid (v_x, v_y) \notin E, x < y}$.

Next, we will define an instance of \prbinf. The set of time series $T =
\enpr{t_1}{t_n} \cup \set{x}$ consists of $n$ time series $t_i$ corresponding to
the node $v_i$, and a single time series $x$ which we will use a seed.
We set the domain to be $F$.
Each time series $t_i$ maps an element of $e = (v_x, v_y) \in F$ to an integer,
\[
	t_i(e) = 1, \ \text{if}\  i = x,\quad
	t_i(e) = -1, \ \text{if}\  i = y,\quad
	t_i(e) = 0, \ \text{otherwise}\quad. 
\]
We also set $x = 0$.
First note that since $t_i(e)$ is an integer between $-1$ and $1$,
the score $\infscore{U}$ is either $0$, $1$, or $2$ for any $U \subseteq T$.

Since we do not have any fully-connected nodes in $G$, there is no non-zero $t_i$
in $T$. Since $x \in U$ for any solution of \prbinf, 
then $\infscore{U} = 0$ implies $\abs{U} = 1$.

Let $W \subseteq V$ be a subset of nodes, and let $U$ be the corresponding time series.
We claim that $\infscore{U} = 1$ if and only if $W$ is a clique. To prove the claim,
first observe that if $v_i, v_j \in W$ such that $e = (v_i, v_j) \in F$, then $t_i(e) = 1$
and $t_j(e) = -1$, thus $\infscore{U} = 2$.  On the other hand, if $W$ is a clique, then
for every $t_i, t_j \in U$ and $e \in F$ such that $t_i(e) \neq 0$, we have $t_j(e) = 0$
since otherwise $(v_i, v_j) \notin E$.
Thus, $\infscore{U} = 1$ if and only if $W$ is a clique.

Let $O$ be the solution for $\prbinf(T, k + 1, x)$.
Note that $\infscore{O} = 1$ if and only if $G$ has a $k$-clique,
and $\infscore{O} = 2$ otherwise.

Let $S$ be the output of $\alpha$-approximation algorithm.
Since $k > 1$, we know that $\infscore{O}$ is either 1 or 2.
If $\infscore{O} = 2$, then $\infscore{S} = 2$.
If $\infscore{O} = 1$, then $\infscore{S} \leq \alpha \infscore{O} < 2 \times 1$.
Thus, $\infscore{S} = 1$. In summary, $\infscore{O} = \infscore{S}$.

We have shown that $\infscore{S} = 1$ if and only if $G$ has a $k$-clique.
This allows us to detect $k$-clique in $G$ in polynomial time proving our claim.
\qed
\end{proof}

%% file: related.tex
\section{Related work}\label{sec:related}

Confidence bands are envelopes for which confidence intervals of individual
points hold simultaneously. \citet{davison1997bootstrap,mandel2008simultaneous} proposed a non-parametric approach for finding simultaneous confidence intervals.
Here, time series
are ordered based on its \emph{maximum} value, and $\alpha$-confidence intervals are obtained
by removing $\alpha / 2$ portions from each tail. Note that unlike \prbsum and \prbinf
this definition is not symmetric: if we flip the sign of time series we may get a different interval.

There is a strong parallel between finding regularized bands and finding
dense subgraphs. Proposition~\ref{prop:sparsest} states that the inner-most
regularized band has the smallest average envelope area, or alternatively it
has the highest ratio of time series per envelope area. A related
graph-theoretical concept is a dense subgraph, a subgraph $H$ of a given subgraph
$G$ with the largest ratio $\abs{E(H)}/\abs{V(H)}$. The method proposed by~\citet{Goldberg:1984up} for finding dense
subgraphs in polynomial time is based on maximizing $\abs{E(H)} - \alpha \abs{V(H)}$
and selecting $\alpha$ to be as small as possible without having an empty solution.
Moreover, \citet{tatti2019density} extended the notion of dense subgraphs to density-friendly core decomposition,
which essentially consists of the subgraphs minimizing $\abs{E(H)} - \alpha \abs{V(H)}$
for various values of $\alpha$, the algorithm for finding the decomposition is
similar to the algorithm for enumerating all regularized bands. In addition,~\citet{tsourakakis2015k}
extended the notion of dense subgraphs to triangle-density and hypergraphs, and also used minimum cut to find the solutions. As pointed out
in Section~\ref{sec:sum} is that we can view time series as sets of points in $P$.
In fact, the minimum cut used in Section~\ref{sec:reg} share some similarities with
the minimum cut proposed by~\citet{tsourakakis2015k}. Finally, the algorithm proposed
by~\citet{korpela2014confidence} to find confidence bands resembles the algorithm by~\citet{Charikar:2000tg} for
approximating the densest subgraph: in the former we delete the time series that
reduce the envelope area the most while in the latter we delete vertices that have the smallest degree.

We assume that we are given a seed time series $x$. If such series is not given
then we need to test every $t \in T$ as a seed. If we consider a special case of $k = 2$,
then the problem of finding regularized band reduces to the closest pair problem: find two time
series with the smallest distance: a well-studied problem in computational geometry. A classic
approach by~\citet{rabin1976probabilistic,dietzfelbinger1997reliable,khuller1995simple} allows to solve the closest pair problem in $\bigO{n}$ time but the analysis treats
the size of the domain, $m$, as a constant; otherwise, the computational complexity has an exponential factor
in $m$ and can be only used for very small values of $m$. For large values of $m$, \citet{indyk2004closest} proposed an
algorithm for solving the closest pair problem minimizing
$\onescore{\cdot}$ in $\bigO{n^{2.687}}$ time
and minimizing
$\infscore{\cdot}$ in $\bigO{n^{2.687}\log \Delta}$ time, where $\Delta$ is the width of the envelope of the whole data.

%% file: exp.tex
\section{Experimental evaluation}\label{sec:exp}

In this section we describe our experimental evaluation.

We implemented \algreg and \algsum using C++ and used a laptop with Intel Core i5 (2.3GHz)
to conduct our experiments.\!\footnote{The code is available at \url{https://version.helsinki.fi/DACS}}
As a baseline we used the algorithm by~\citet{korpela2014confidence}, which we will call \algpeel.
We implemented \algpeel also with C++, and modified it to make sure that the seed time series $x$ is always included.
Finally, we implemented \alginf with Python.
In all algorithms we used the median as the seed time series.

\begin{table}[th!]

\caption{Basic characteristics of the datasets and performance measures of the
algorithms. Here, $n$ stands for the number of time series, $m$ stands for the
domain size, $\abs{B_1}$ is the size of the smallest non-trivial regularized
band, $\abs{\mathcal{B}}$ is the number of regularized bands, and time is the required
time to execute \algreg in seconds.
The scores $\onescore{}$
for the algorithms \algsum, \alginf, and \algpeel are normalized with the envelope area of the whole
data and multiplied by 100. 
}
\label{tab:stats}

\begin{filecontents*}{results.csv}
name,n,m,n90,n95,findreg90,peel90,max90,findreg95,peel95,max95,maxarea,innersize,regcnt,findreg90r,peel90r,max90r,findreg95r,peel95r,max95r,time
\dtname{Milan},245,12,220,232,203.373195,209.590015,214.2310705426,217.733026,222.600768,226.8217003428,289.115361,209,17,70.3432686165714,72.4935590675862,74.0988198626361,75.3100856512429,76.9937533689191,78.4537008197223,0.029
\dtname{Power},1417,24,1275,1346,81.346467,83.504400,88.354966666679,90.452433,93.069733,94.377626436779,114.661533,1102,56,70.9448625634545,72.8268651353196,77.0571998777297,78.8864675304838,81.1690987944492,82.3097546033847,3.683
\dtname{ECG-normal},1507,253,1356,1431,160.240000,161.115000,226.085,177.300000,178.190000,226.65,309.830000,1289,72,51.718684,52.001097,72.970661330407,57.224930,57.512184,73.1530193977342,39.439
\dtname{ECG-pvc},520,253,468,494,371.190000,369.970000,425.245,387.650000,387.995000,442.435,462.355000,484,19,80.282467,80.018600,91.9736998626596,83.842502,83.917120,95.6916222383234,6.668
\end{filecontents*}
\pgfplotstabletypeset[
    begin table={\begin{tabular*}{\textwidth}},
    end table={\end{tabular*}},
    col sep=comma,
	columns = {name, n, m, innersize, regcnt, time, findreg90r, peel90r, max90r, findreg95r, peel95r, max95r},
    columns/name/.style={string type, column type={@{\extracolsep{\fill}}l}, column name=\emph{Dataset}},
    columns/n/.style={fixed, set thousands separator={\,}, column type=r, column name=$n$},
    columns/m/.style={fixed, set thousands separator={\,}, column type=r, column name=$m$},
    columns/innersize/.style={fixed, set thousands separator={\,}, column type=r, column name=$\abs{B_1}$},
    columns/findreg90r/.style={fixed, set thousands separator={\,}, column type=r, column name=$\algsums$},
    columns/findreg95r/.style={fixed, set thousands separator={\,}, column type=r, column name=$\algsums$},
    columns/peel90r/.style={fixed, set thousands separator={\,}, column type=r, column name=$\algpeel$},
    columns/peel95r/.style={fixed, set thousands separator={\,}, column type=r, column name=$\algpeel$},
    columns/max90r/.style={fixed, set thousands separator={\,}, column type=r, column name=$\alginfs$},
    columns/max95r/.style={fixed, set thousands separator={\,}, column type=r, column name=$\alginfs$},
    columns/regcnt/.style={fixed, set thousands separator={\,}, column type=r, column name=$\abs{\mathcal{B}}$},
    columns/time/.style={fixed, set thousands separator={\,}, column type=r, column name={Time}},
    every head row/.style={
		before row={
			\toprule
			\arraybackslash
			& & & & & & \multicolumn{3}{l}{$\onescore{}$ for $k = \floor{0.9n}$} & \multicolumn{3}{l}{$\onescore{}$ for $k = \floor{0.95n}$} \\
			\cmidrule(r){7-9}
			\cmidrule(r){10-12}},
			after row=\midrule},
     every last row/.style={after row=\bottomrule},
]
{results.csv}

\end{table}

\begin{table}[th!]

\caption{Scores $\infscore{}$ of discovered confidence bands.
The scores are normalized with the envelope width of the whole
data and multiplied by 100.
}
\label{tab:stats2}

\begin{filecontents*}{results2.csv}
name,findreg90,peel90,max90,findreg95,peel95,max95,maxarea,findreg90r,peel90r,max90r,findreg95r,peel95r,max95r
\dtname{Milan},21.645161,23.806452,20.0967741936,23.285714,23.806452,22.0645161291,29.838710,72.5405387833455,79.7837842185537,67.3513506234016,78.0386082374205,79.7837842185537,73.9459451467573
\dtname{Power},5.081967,5.279800,4.699866666663,5.279800,6.343700,5.081966666663,6.426333,79.0803557798826,82.1588299268027,73.1345024707403,82.1588299268027,98.7141500448234,79.0803505928342
\dtname{ECG-normal},1.885000,1.885000,1.595,1.910000,1.885000,1.67,2.910000,64.776632,64.776632,54.8109965635739,65.635739,64.776632,57.3883161512027
\dtname{ECG-pvc},4.275000,4.275000,3.045,4.275000,4.275000,3.755,4.585000,93.238822,93.238822,66.412213740458,93.238822,93.238822,81.8974918211559
\end{filecontents*}
\pgfplotstabletypeset[
    begin table={\begin{tabular*}{\textwidth}},
    end table={\end{tabular*}},
    col sep=comma,
	columns = {name, findreg90r, peel90r, max90r, findreg95r, peel95r, max95r},
    columns/name/.style={string type, column type={@{\extracolsep{\fill}}l}, column name=\emph{Dataset}},
    columns/findreg90r/.style={fixed, set thousands separator={\,}, column type=r, column name=$\algsums$},
    columns/findreg95r/.style={fixed, set thousands separator={\,}, column type=r, column name=$\algsums$},
    columns/peel90r/.style={fixed, set thousands separator={\,}, column type=r, column name=$\algpeel$},
    columns/peel95r/.style={fixed, set thousands separator={\,}, column type=r, column name=$\algpeel$},
    columns/max90r/.style={fixed, set thousands separator={\,}, column type=r, column name=$\alginfs$},
    columns/max95r/.style={fixed, set thousands separator={\,}, column type=r, column name=$\alginfs$},
    every head row/.style={
		before row={
			\toprule
			\arraybackslash
			& \multicolumn{3}{l}{$\infscore{}$ for $k = \floor{0.9n}$} & \multicolumn{3}{l}{$\infscore{}$ for $k = \floor{0.95n}$} \\
			\cmidrule(r){2-4}
			\cmidrule(r){5-7}},
			after row=\midrule},
     every last row/.style={after row=\bottomrule},
]
{results2.csv}

\end{table}

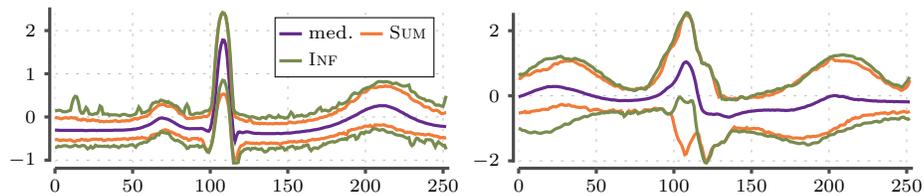
\begin{figure}[th!]

\begin{tikzpicture}
\begin{axis}[
	width = 5.2cm,
    height = 2cm,
    cycle list name=yaf,
    scale only axis,
    tick scale binop=\times,
    x tick label style = {/pgf/number format/set thousands separator = {\,}},
    y tick label style = {/pgf/number format/set thousands separator = {\,}},
    scaled ticks = false,
    legend pos = north east,
	legend entries = {med., \algsums, \alginfs},
	legend columns=2,
    no markers]
\addplot+[yafcolor1] table[x expr=\coordindex, y index = 0, header = false] {normal_ecg.env};
\addplot+[yafcolor2] table[x expr=\coordindex, y index = 1, header = false] {normal_ecg.env};
\addplot+[yafcolor2, forget plot] table[x expr=\coordindex, y index = 2, header = false] {normal_ecg.env};
\addplot+[yafcolor3] table[x expr=\coordindex, y index = 3, header = false] {normal_ecg.env};
\addplot+[yafcolor3, forget plot] table[x expr=\coordindex, y index = 4, header = false] {normal_ecg.env};
\pgfplotsextra{\yafdrawaxis{1}{253}{-1}{2.5}}
\end{axis}
\end{tikzpicture}
\begin{tikzpicture}
\begin{axis}[
	width = 5.2cm,
    height = 2cm,
    cycle list name=yaf,
    scale only axis,
    tick scale binop=\times,
    x tick label style = {/pgf/number format/set thousands separator = {\,}},
    y tick label style = {/pgf/number format/set thousands separator = {\,}},
    scaled ticks = false,
    no markers]
\addplot+[yafcolor1] table[x expr=\coordindex, y index = 0, header = false] {pvc_ecg.env};
\addplot+[yafcolor2] table[x expr=\coordindex, y index = 1, header = false] {pvc_ecg.env};
\addplot+[yafcolor2] table[x expr=\coordindex, y index = 2, header = false] {pvc_ecg.env};
\addplot+[yafcolor3] table[x expr=\coordindex, y index = 3, header = false] {pvc_ecg.env};
\addplot+[yafcolor3] table[x expr=\coordindex, y index = 4, header = false] {pvc_ecg.env};
\pgfplotsextra{\yafdrawaxis{1}{253}{-2}{2.5}}
\end{axis}
\end{tikzpicture}

\caption{Envelopes for \dtname{ECG-normal} (left) and \dtname{ECG-pvc} (right) and $k = \floor{0.9n}$.}
\label{fig:ecg}

\end{figure}
\textbf{Datasets:}
We used 4 real-world datasets as benchmark datasets.
The first dataset, \dtname{Milan}, consists of monthly averages of maximum daily temperatures in Milan
between the years 1763--2007.\!\footnote{\url{https://www.ncdc.noaa.gov/}}
The second dataset, \dtname{Power}, consists of hourly power consumption
(variable \texttt{global\_active\_power}) of a single
household over almost 4 years, a single time series representing a day.\!\footnote{\url{http://archive.ics.uci.edu/ml/datasets/Individual+household+electric+power+consumption}}
Our last 2 datasets \dtname{ECG-normal} and \dtname{ECG-pvc} are heart beat data~\citep{goldberger2000physiobank}.
We used MLII data of a single patient (id 106) from the MIT-BIH arrhythmia database,\!\footnote{\url{https://physionet.org/content/mitdb/1.0.0/}} and split the measurements
into normal beats (\dtname{ECG-normal}) and abnormal beats with premature ventricular contraction
(\dtname{ECG-pvc}). Each time series represent measurements between $-300$ms and $400$ms around each beat.
The sizes of the datasets are given in Table~\ref{tab:stats}.

\textbf{Results:}
First let us consider \algreg. From the results given Table~\ref{tab:stats} we
see that the number of distinct regularized bands $\abs{\mathcal{B}}$ is low:
about 4\%--7\% of $n$, the number of time series. Having so few bands in practice reduces
the computational cost of \algreg since the algorithm tests at most $2\abs{\mathcal{B}}$
values of $\alpha$
Interestingly, the smallest non-trivial band $B_1$ is typically large,
containing about 70\%--90\% of the time series. Note that Proposition~\ref{prop:sparsest} states that $B_1$ 
has the smallest ratio of $\onescore{B_1} / \abs{B_1}$. For our benchmark datasets, $B_1$ is large suggesting
that most time series are equally far away from the median while the remaining the time series exhibit outlier behaviour.

The algorithms are fast for our datasets: Table~\ref{tab:stats}
show that \algreg requires at most 40 seconds. Additional steps required by \algsum are negligible, completing in less than a second.
The baseline algorithm is also fast, requiring less than a second to complete.

Let us now compare \algsum against \algpeel. We compared the obtained areas by both algorithms with $k = \floor{0.9n}$ 
and  $k = \floor{0.95n}$.
We see from the results in Table~\ref{tab:stats}, that \algsum performs slightly better than \algpeel.
The improvement in score is modest, 1\%--2\%. We conjecture that in practice \algpeel is close
to the optimal, so any improvements are subtle. Interestingly, enough \algpeel performs better than \algsum
for \dtname{ECG-pvc} and $\gamma = 0.1$. The reason for this is that the inner band $B_1$ contains
more than 90\% of the time series. In such a case \algsum will reduce to a simple greedy method, starting
from $\set{x}$. Additional testing revealed that \algpeel outperforms \algsum when $k \leq \abs{B_1}$ about 50\%--90\%, depending on the dataset,
suggesting that whenever $k \leq \abs{B_1}$ it is probably better to run both algorithms and select the better envelope.

Next let us compare \alginf against the other methods. The results in
Tables~\ref{tab:stats}--\ref{tab:stats2} show that \alginf yields inferior
$\onescore{}$ scores but superior $\infscore{}$ scores. This is expected as
\alginf optimizes  $\infscore{}$ while \algsum and \algpeel optimize
$\onescore{}$. The differences are further highlighted in the envelopes for \dtname{ECG} datasets shown in Figure~\ref{fig:ecg}:
\alginf yields larger envelopes but provides a tighter bound under the peak (R wave).

%% file: conclusions.tex
\section{Concluding remarks}\label{sec:conclusions}

In this paper we consider the approximation algorithms for
discovering confidence bands. Namely, we proposed a practical algorithm
that approximates \prbsum with a guarantee of $\bigO{n^{1/2}}$.
We also argued that the lower bound for the guarantee is most likely
$\bigO{n^{1/4}}$. In addition, we showed that we can 2-approximate
\prbinf, a variant of \prbsum problem, with a simple algorithm and that the guarantee is tight.

Our experiments showed that \algsum outperforms the original baseline method for
large values of $k$, that is, as long as $k$ is larger than the smallest
regularized band.  Our experiments suggest that this condition usually  holds, if we are
interested, say in, 90\%--95\% confidence.

Interesting future line of work is to study the case for time series with multiple modes,
that is, a case where instead of a single seed time series, we are given a set of time series,
and we are asked to find confidence bands around each seed.